\documentclass[11pt]{article}

\usepackage{fullpage}
\usepackage{amsmath,amssymb,amsthm}
\usepackage{mathtools}
\usepackage{booktabs}
\usepackage{tikz}
\usetikzlibrary{arrows.meta}
\usepackage[colorlinks=true,linkcolor=blue,citecolor=blue,urlcolor=blue]{hyperref}
\usepackage[font={small,it}]{caption}

\theoremstyle{plain}
\newtheorem{theorem}{Theorem}[section]
\newtheorem{lemma}[theorem]{Lemma}
\newtheorem{proposition}[theorem]{Proposition}
\newtheorem{corollary}[theorem]{Corollary}
\theoremstyle{definition}
\newtheorem{definition}[theorem]{Definition}

\theoremstyle{remark}
\newtheorem{remark}[theorem]{Remark}

\DeclareMathOperator{\tr}{tr}
\DeclareMathOperator{\Reg}{Reg}
\DeclareMathOperator{\range}{range}
\DeclareMathOperator{\spn}{span}

\DeclareMathOperator*{\argmax}{arg\,max}
\DeclareMathOperator*{\argmin}{arg\,min}
\newcommand{\R}{\mathbb{R}}

\newcommand{\Ball}{\mathbb{B}}

\newcommand{\wstar}{w^{\ast}}
\newcommand{\xstar}{x^{\ast}}

\newcommand{\ip}[2]{\langle #1, #2\rangle}

\title{\vspace{-1.5cm}\bf Efficient Online Inverse Optimization with $O(d)$
Regret\thanks{The main result was entirely obtained by Cogentic, an agentic framework for mathematical discovery, using Gemini 3.1 Pro as the base model. There was no human interaction other than the initial \href{https://drive.google.com/file/d/1fCq4ERodRruPHnLQLf7bsSfA3gnGB_2A/view}{prompt} for cogentic to generate a \href{https://drive.google.com/file/d/1CzXV9MPSkg_nqrCusJBLjZaWDjNJGzNW/view?usp=drive_link
}{full proof}. The authors contextualized the findings, verified the proofs, and extended the applications. The full exposition here is due to the authors aided by different AI models.

The Cogentic harness was built by Yang Cai, Vineet Gupta, Yanchen Jiang,
Christopher Liaw, Aranyak Mehta, Grigoris Velegkas and Di Wang. The problem was proposed by Anupam Gupta, Guru Guruganesh and Renato Paes Leme, who also led the exposition, investigated provenance of claims and extended to further applications beyond the basic result.

The following authors have additional affiliations beyond Google Research. Yang Cai (Yale University), Anupam Gupta (New York University) and Vineet Gupta (Google DeepMind).
}}

\author{
\parbox{0.95\textwidth}{\centering
Yang Cai,
Anupam Gupta,
Vineet Gupta,
Guru Guruganesh,
Yanchen Jiang,
Christopher Liaw,
Aranyak Mehta,
Renato Paes Leme,
Grigoris Velegkas,
Di~Wang
}\\[1.5em]
\normalsize Google Research
}
\date{}

\begin{document}
\maketitle
\vspace{-1.2em}


\begin{abstract}
We give a deterministic algorithm for online inverse linear optimization with regret $O(d)$, uniform
in the horizon and $O(d^{2})$ time per round. A bound of this order was obtained recently by
Dewasurendra, settling a question of Gollapudi et al.\ and of Oki and Sakaue, but by an improper
rule that enumerates covers at every scale and costs $T^{\Theta(d)}$ a round; ours is the first
efficient such bound and the first proper one. We build on the variable-metric framework of Sakaue
et al., adding a self-normalized rank-one update, and we replace the $\log\det$ potential by the
trace power $\tr(H^{-1/2})$, which is bounded outright and removes the $\ln T$. The bound also holds
against an expert that does not optimize, and we give corruption-robust and rank-adaptive variants,
and an application to convex minimization.
\end{abstract}

\section{Introduction}\label{sec:intro}

A medical resident learns by watching the decisions of an expert attending physician. Each patient
presents a set of available interventions --- which drug, whether to operate, whether to wait ---
and the attending chooses one. What the attending is optimizing is largely tacit: an accumulated
weighting of efficacy against side effects, of recovery time against cost. The resident never sees
that criterion, only the choice, after the fact, patient after patient. This is the problem of
online inverse linear optimization.

We model the criterion as a fixed vector $\wstar$: interventions are described by vectors of
attributes, those available for the $t$-th patient forming a set $X_t\subseteq\R^d$. The resident
recommends an intervention and then sees the choice of the attending, who selects one maximizing
$\ip{\wstar}{\cdot}$ --- and nothing else: not $\wstar$, not the value of any intervention, not even
how much worse the recommendation was. This motivates the following optimization problem.

\smallskip
\noindent\textit{Online inverse linear optimization.} Fix a hidden $\wstar\in\Ball$. For
$t=1,2,\dots$:
\begin{enumerate}
\itemsep0pt \parskip0pt \topsep2pt
\item a compact nonempty action set $X_t\subseteq\Ball$ is revealed, chosen adversarially as a
      function of the history;
\item the learner recommends an action $\hat x_t\in X_t$;
\item the expert takes $x_t\in\argmax_{x\in X_t}\ip{\wstar}{x}$, which the learner then observes.
\end{enumerate}
\smallskip

Here $\Ball$ is the unit ball of $\R^d$; the normalizations fix the scale. We measure performance by
the cumulative shortfall of the recommendations, in the expert's own currency:
\begin{equation}\label{eq:regret}
  R_T\;:=\;\sum_{t=1}^{T}\ip{\wstar}{x_t-\hat x_t}\;\ge\;0,
\end{equation}
where each summand is nonnegative by optimality of $x_t$. A learner with small $R_T$ recommends
about as well as the expert without ever identifying $\wstar$, which is anyway out of reach:
objectives inducing the same choices are never separated by data.

What matters is the dependence on the horizon $T$. A bound of $O(\sqrt T)$, or even $O(d\log T)$,
makes the \emph{average} shortfall vanish but lets the total grow without bound: over a long enough
career the resident keeps paying. A bound \emph{uniform in $T$} says instead that the total damage
ever done is finite --- mistakes are not merely rare but exhausted. That is the guarantee we are
after, with a constant polynomial in $d$.

One more distinction will matter. Call a learner \emph{proper} if it commits to an objective before
it sees the menu: at the start of round $t$ it holds an estimate $\hat w_t\in\R^d$, $\hat w_t\neq0$,
of $\wstar$, and once $X_t$ arrives it recommends
\[
  \hat x_t\;\in\;\argmax_{x\in X_t}\ip{\hat w_t}{x}.
\]
A proper learner thus answers not only ``what should be done for this patient'' but ``what do I take
the attending to be optimizing''; its recommendation is by construction optimal for some objective,
and producing it costs one linear optimization over $X_t$ and nothing else. A learner is
\emph{improper} if $\hat x_t$ may be any function of $X_t$ and the history, with no objective behind
it.

\begin{table}[t]
\centering\scriptsize
\begin{tabular}{@{}lllll@{}}
\toprule
Method & Regret & Uniform in $T$ & Proper & Per-round cost \\
\midrule
B\"armann et al.\ \cite{Barmann17,Barmann20} (OGD/MWU) & $O(\sqrt T)$ & no & yes & $O(\tau_{\mathrm{sol}}+d)$ \\
Besbes et al.\ \cite{Besbes21,Besbes25} (circumcenter) & $O(d^{4}\ln T)$ & no & yes & $\mathrm{poly}(d,T)$ \\
Gollapudi et al.\ \cite{GGKMPS} (centroid, scale $1/T$) & $O(d\ln T)$ & no & yes & est.\ $O(\tau_{\mathrm{sol}}+d^{5}T^{3})$ \\
Gollapudi et al.\ \cite{GGKMPS} (John ellipsoid) & $\exp(O(d\ln d))$ & \textbf{yes} & yes & $\mathrm{poly}(d,T)$ \\
Sakaue et al.\ \cite{SakaueONS} (online Newton step) & $O(d\ln T)$ & no & yes &
  $O(\tau_{\mathrm{sol}}+d^{2}+\tau_{\mathrm{proj}})$ \\
Sakaue \cite{SakaueSOP} (second-order perceptron) & $O(d\ln T)$ & no & yes & $O(\tau_{\mathrm{sol}}+d^{2})$ \\
Dewasurendra \cite{Dewasurendra26} (multiscale vote) & $O(d)$ & \textbf{yes} & no & $T^{\Theta(d)}$ \\
\midrule
\textbf{This paper} & $\mathbf{O(d)}$ & \textbf{yes} & \textbf{yes} & $O(\tau_{\mathrm{sol}}+d^{2})$ \\
\addlinespace
Lower bound (Theorem~\ref{thm:lower}; cf.\ \cite{Dewasurendra26,SakaueONS}) & $\Omega(\sqrt d)$ & --- & --- & --- \\
\bottomrule
\end{tabular}
\caption{Online inverse linear optimization; $\tau_{\mathrm{sol}}$ is one linear optimization and
$\tau_{\mathrm{proj}}$ one Mahalanobis projection onto the domain, $O(d^{3})$ for the ball
\cite{SakaueONS}. The count $T^{\Theta(d)}$ is from \cite{Dewasurendra26} itself; there the
per-round work is the maximization of a nonconvex vote over $X_t$, for which a linear optimization
oracle does not suffice.}
\label{tab:oilo}
\vspace{-1em}
\end{table}

\paragraph{History.} The online formulation is due to B\"armann, Pokutta and Schneider
\cite{Barmann17,Barmann20}, who analyzed online gradient descent and multiplicative weights and
obtained $O(\sqrt T)$. The same protocol was studied independently, under the name \emph{contextual
recommendation}, by Gollapudi, Guruganesh, Kollias, Manurangsi, Paes Leme and Schneider
\cite{GGKMPS}. Besbes, Fonseca and Lobel \cite{Besbes21,Besbes25} gave a circumcenter rule with
$O(d^{4}\ln T)$. Gollapudi et al.\ obtained $O(d\ln T)$ from a regularized center of gravity and
$\exp(O(d\ln d))$ from the John ellipsoid, the latter uniform in $T$ but exponential in $d$. Sakaue,
Tsuchiya, Bao and Oki \cite{SakaueONS} obtained $O(d\ln T)$ efficiently by an online Newton step,
together with an $\Omega(\sqrt d)$ lower bound for every horizon (Table~\ref{tab:oilo}). Sakaue
\cite{SakaueSOP} then removed the Mahalanobis projection that step requires, reaching the same
$O(d\ln T)$ at $O(d^{2})$ per round by a second-order perceptron update. Two further
results are finite and uniform in $T$ under extra assumptions:
\cite{SakaueFY} needs a value gap separating the optimal action from the rest in every round, and
\cite{OkiSakaue} needs every $X_t$ to be $\mathrm{M}$-convex, a discrete-convexity condition met by
matroid families but not in general. For \emph{arbitrary} action sets the choice has been between
$O(d\log T)$, polynomial but growing with the horizon, and $\exp(O(d\log d))$, uniform in $T$ but
exponential in $d$; whether a finite $\mathrm{poly}(d)$ bound is achievable was left open in both
\cite{GGKMPS} and \cite{OkiSakaue}. Every rule named so far is proper.

That question was answered recently by Dewasurendra
\cite{Dewasurendra26},\footnote{The same paper proves an $\Omega(d)$ lower bound, stated against the
entropy exponent rather than the Euclidean normalization used here: its hard instance takes
$\wstar\in\{-1,1\}^{d}$, so $\|\wstar\|=\sqrt d$, and rescaling it into $\Ball$ turns the bound into
$\sqrt d/2$, the same $\Omega(\sqrt d)$ as Theorem~\ref{thm:lower}.} by a rule that is neither proper
nor efficient. At every dyadic scale he covers the class of optimality-gap functions, turning each
cover element into a tolerant test that accepts or rejects an action; the tests at all scales are
pooled into one weighted vote, the learner plays the action the vote likes best, and the expert's
choice then drops the tests rejecting it and promotes those separating it from the recommendation.
A Dudley entropy integral bounds the shortfall by $O(d)$. The vote carries $T^{\Theta(d)}$ tests and
is nonconvex over $X_t$, so a round costs $T^{\Theta(d)}$ and no linear optimization oracle helps;
and its maximizer need not maximize any $\ip{w}{\cdot}$, so the recommendation carries no estimate
of the objective.

Online inverse optimization is the online face of an older problem --- recover an objective from
observed optimal decisions --- initiated by Burton and Toint \cite{BurtonToint} for shortest paths
and cast in general linear-programming form by Ahuja and Orlin \cite{AhujaOrlin}; see
\cite{ChanSurvey} for a survey. The offline literature turns largely on the choice of loss --- KKT residuals
\cite{Keshavarz11}, equilibrium losses \cite{Bertsimas15}, suboptimality \cite{MohajerinEsfahani18}
and incenter \cite{Zattoni25} costs --- and on noisy \cite{Aswani18} or partially specified
\cite{Ren25} data, with statistical rates only recently established \cite{Fatemi26}; applications
run from control \cite{Akhtar21} and routing \cite{ZattoniRouting} to offline reinforcement
learning \cite{Dimanidis25}.

\paragraph{Reduction to a cutting-plane problem.} The protocol above is not itself a cutting-plane
problem --- the learner outputs an action and receives an action, not a query point and a halfspace
--- but it reduces to one. The primitive, isolated in \cite{GGKMPS}, is the following.

\smallskip
\noindent\textit{Cutting planes with a strong separation oracle.} A point $\wstar$ is hidden in a
known compact convex body $K_1\subseteq\R^d$. For $t=1,2,\dots$:
\begin{enumerate}
\itemsep0pt \parskip0pt \topsep2pt
\item the learner queries a point $p_t\in\R^d$;
\item an adversary, having seen $p_t$ and the history, returns a unit vector $v_t$ subject only to
      $\ip{\wstar-p_t}{v_t}\ge0$.
\end{enumerate}
The regret is $\Reg_T:=\sum_{t\le T}\ip{\wstar-p_t}{v_t}$, again a sum of nonnegative terms; the
oracle is \emph{strong} in that $v_t$ may be chosen after seeing $p_t$, where a \emph{weak} one
commits to $\pm v_t$ in advance.
\smallskip

To perform the reduction, assume that we have a low-regret (proper) learner $\mathcal{L}$ for the cutting-plane game. (We restrict to proper learners here --- every algorithm of Table~\ref{tab:oilo} except \cite{Dewasurendra26} is one.) We use it to solve inverse
optimization as follows: at each time $t$, we use the query point
$p_t$ maintained by the learner $\mathcal{L}$ at time $t$ as our 
current estimate for the
unknown weight function; i.e., we define $\hat w_t := p_t$ (if $p_t=0$ ---
for our algorithm only at $t=1$ --- take any fixed nonzero $\hat w_t$
instead; then $\ip{p_t}{\delta_t}=0$ and the argument below is
unchanged). We then
play the current optimizer
$\hat{x}_t \in \arg\max_{x \in X_t} \ip{\hat w_t}{x}$; the expert
responds with $x_t \in \arg\max_{x \in X_t} \ip{\wstar}{x}$. Upon
this, we define $\delta_t := x_t-\hat x_t$ and
$v_t:=\delta_t/\|\delta_t\|$ (rounds with $\delta_t=0$ cost nothing and
are skipped). Optimality of $x_t$ for $\wstar$ gives
$\ip{\wstar}{\delta_t}\ge0$ and that of $\hat x_t$ for $\hat w_t$
gives $\ip{\hat w_t}{\delta_t}\le0$; subtracting the two,
\begin{equation}
  \ip{\wstar-\hat w_t}{\delta_t} \geq 0, 
\end{equation}
which means that $\ip{\wstar-p_t}{v_t} \geq 0$, making this a legal
round of the cutting-plane game. Next, observe that 
\begin{equation}\label{eq:reduction}
  \ip{\wstar-\hat w_t}{\delta_t} 
  \;=\; \ip{\wstar}{\delta_t}
  \;-\;\underbrace{\ip{\hat w_t}{\delta_t}}_{\leq\,0} \geq
    \ip{\wstar}{\delta_t} = \ip{\wstar}{x_t-\hat x_t};
\end{equation}
since $x_t,\hat x_t\in\Ball$ we have $\|\delta_t\|\le2$,
summing~(\ref{eq:reduction}) over all timesteps gives
\begin{equation}
  \label{eq:transfer}
  R_T = \sum_t \ip{\wstar}{x_t - \hat x_t} \leq \sum_t \ip{\wstar -
    \hat w_t}{\delta_t} = \sum_t \| \delta_t \| \, \ip{\wstar -
    p_t}{v_t} \leq 2 \Reg_T.
\end{equation}
The middle quantity is the symmetrized regret
$\tilde R_T:=\sum_t\ip{\hat w_t-\wstar}{\hat x_t-x_t}\ge R_T$ that part of the literature bounds. A
cutting-plane algorithm with regret at most $B$ on $K_1=\Ball$ therefore yields a learner with
$\tilde R_T\le2B$, hence $R_T\le2B$, for every $T$. The oracle really is strong: the
adversary picks $X_t$ knowing $\hat w_t$, and $v_t$ depends on $\hat x_t$, so $v_t$ follows $p_t$
and weak-oracle guarantees do not apply. It therefore suffices to bound the regret of this game by
$\mathrm{poly}(d)$, uniformly in $T$. The rest of the paper does so --- with a method that is not a
cutting-plane method at all.

\paragraph{Competing against a weaker expert.} Optimality of the expert was used only through the
inequality $\ip{\wstar}{\delta_t}\ge0$, and only to make the round legal. So the expert need not
be an optimizer at all: it is enough that the attending physician's intervention is, round by round,
at least as good as the resident's under the attending's own criterion. Everything above then holds
verbatim, and $R_T$ becomes the shortfall against what the attending actually did rather than
against the best available action. We keep the optimizing expert as the definition, since that is
the standard model, but no part of the analysis needs it.

This is the first $O(d)$ bound that competes with an expert that does not optimize: the earlier
analyses at this order use the expert's optimality throughout, and \cite{Dewasurendra26} pays a
multiple of its cumulative suboptimality.

\subsection{Our Results}\label{sec:results}

We solve this cutting-plane game with a bound that does not involve the horizon at all.

\begin{theorem}\label{thm:main}
There is a deterministic algorithm which, given only the radius $R=\max_{w\in K_1}\|w\|$,
satisfies
\[
  \Reg_T\;=\;\sum_{t=1}^{T}\ip{\wstar-p_t}{v_t}\;\le\;6\sqrt3\;d\,R\;\le\;11\,d\,R
\]
for every $T$, every hidden $\wstar\in K_1$ and every strong separation oracle.
\end{theorem}

The bound is uniform in $T$, hence holds for $T=\infty$, and the algorithm is deterministic and
anytime. Taking $K_1=\Ball$ and combining with \eqref{eq:transfer} gives a deterministic learner for
online inverse linear optimization with $R_T\le\tilde R_T=O(d)$ for every horizon and every
adversarial sequence of action sets, at $O(d^{2})$ arithmetic and one linear optimization per round.
This is the first efficient bound of that order, and the first proper one: the question left open in
\cite{GGKMPS,OkiSakaue} was answered in \cite{Dewasurendra26}, but by an improper rule costing
$T^{\Theta(d)}$ a round. The dimension dependence is within $\sqrt d$ of optimal, since every
algorithm suffers $\Reg\ge\sqrt d\,R$ (Theorem~\ref{thm:lower}).
Section~\ref{sec:convex} records a second consequence: minimizing an $L$-Lipschitz convex function
from subgradient directions alone, with total suboptimality $O(dLR)$ over an infinite run.

\paragraph{The algorithm.} Write $g_t=-v_t$, so that the round's loss is
$r_t=\ip{p_t-\wstar}{g_t}\ge0$: the learner is running online gradient descent on the linear losses
$w\mapsto\ip{w}{g_t}$, with the hidden point $\wstar$ as a comparator that is never beaten.

Our algorithm builds on the framework of Sakaue, Tsuchiya, Bao and Oki \cite{SakaueONS}, who cast
the problem as a variable-metric online gradient descent and run an online Newton step on these
losses, obtaining $O(d\ln T)$ at $O(d^{2})$ per round plus a Mahalanobis projection onto the
learner's domain. Sakaue \cite{SakaueSOP} later showed that the projection can be dispensed with
altogether --- the recommendation is invariant under positive rescaling of the estimate, so the
estimate needs no domain to be confined to --- and that a second-order perceptron update on the
residuals then costs $O(d^{2})$ per round outright. The regret, however, remains $O(d\ln T)$, the
logarithm arising there from the same log-determinant potential. We too keep no domain and project
nowhere, and we match that per-round cost. What we add is one feature in the algorithm --- a
\emph{self-normalization} step in the metric update --- and one new technique in the analysis ---
bounding a \emph{trace power} of the metric. Together these remove the $\ln T$.

The learner maintains a \emph{metric}, encoded by a positive definite matrix $H_t$ starting from
$H_1=I$, which records the information about $\wstar$ accumulated so far. For any direction $u$, the
quadratic form $u^{\top}H_tu$ measures how much has already been learned along $u$: directions
queried often are stretched and directions not yet explored are left alone. The algorithm then takes
the form
\begin{equation}\label{eq:vm}
  s_t=\|g_t\|_{H_t^{-1}},\qquad
  H_{t+1}=H_t+\frac{\tau}{s_t}\,g_tg_t^{\top},\qquad
  p_{t+1}=p_t-\alpha H_{t+1}^{-1}g_t,
\end{equation}
for constants $\alpha,\tau$ fixed in advance; it never maintains the knowledge set $K_t$, so it is
not a cutting-plane method.

After observing the gradient $g_t$, the algorithm stretches the metric in that direction, recording
the information just acquired. The innovation is the \emph{self-normalization}: we divide the
rank-one update $g_tg_t^{\top}$ by $s_t=\|g_t\|_{H_t^{-1}}$, the dual-metric length of $g_t$.

The step direction $H_{t+1}^{-1}g_t$ has a geometric meaning: it moves aggressively in directions
that are still unexplored (where $H_t$ is close to $I$) and conservatively along directions already
queried (where $H_t$ is large), with the preconditioner also bending the step away from $g_t$
towards whatever remains unknown (Figure~\ref{fig:metric}).

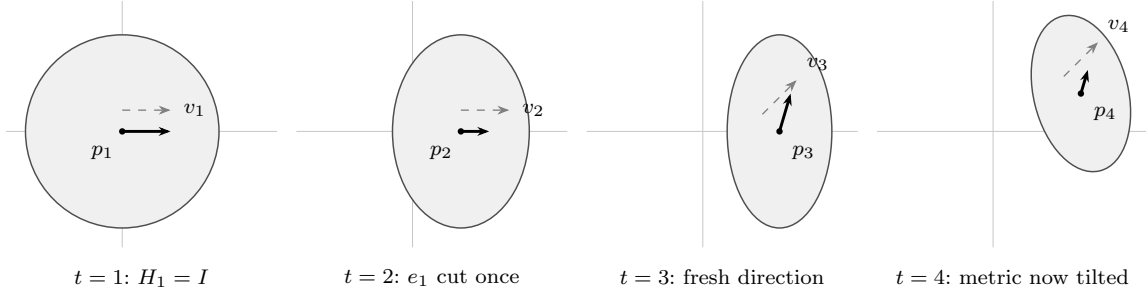
\begin{figure}[t]
\centering
\begin{tikzpicture}[scale=1.28,>={Stealth[length=4.5pt]},
   ax/.style={black!22,line width=0.3pt},
   ell/.style={draw=black!70,fill=black!6,line width=0.55pt},
   mv/.style={->,line width=1.1pt},
   reply/.style={->,dashed,black!50,line width=0.5pt},
   lbl/.style={font=\scriptsize},
   pcap/.style={font=\scriptsize,anchor=north}]
\begin{scope}
  \draw[ax] (-1.2,0)--(1.6,0);  \draw[ax] (0,-1.2)--(0,1.35);
  \draw[ell] (0,0) ellipse (1 and 1);
  \draw[reply] (0,0.22) -- (0.5,0.22);  \node[lbl,anchor=west] at (0.53,0.22) {$v_1$};
  \draw[mv] (0,0) -- (0.5,0);
  \fill (0,0) circle (0.033);           \node[lbl,anchor=north east] at (0.02,-0.04) {$p_1$};
  \node[pcap] at (0.2,-1.32) {$t=1$:\ $H_1=I$};
\end{scope}
\begin{scope}[shift={(3.0,0)}]
  \draw[ax] (-1.2,0)--(1.6,0);  \draw[ax] (0,-1.2)--(0,1.35);
  \draw[ell] (0.5,0) ellipse (0.707 and 1);
  \draw[reply] (0.5,0.22) -- (1.0,0.22); \node[lbl,anchor=west] at (1.03,0.22) {$v_2$};
  \draw[mv] (0.5,0) -- (0.793,0);
  \fill (0.5,0) circle (0.033);          \node[lbl,anchor=north east] at (0.52,-0.04) {$p_2$};
  \node[pcap] at (0.2,-1.32) {$t=2$:\ $e_1$ cut once};
\end{scope}
\begin{scope}[shift={(6.0,0)}]
  \draw[ax] (-1.2,0)--(1.6,0);  \draw[ax] (0,-1.2)--(0,1.35);
  \draw[ell] (0.793,0) ellipse (0.541 and 1);
  \draw[reply] (0.616,0.177) -- (0.970,0.531); \node[lbl,anchor=south west] at (0.97,0.53) {$v_3$};
  \draw[mv] (0.793,0) -- (0.908,0.392);
  \fill (0.793,0) circle (0.033);        \node[lbl,anchor=north west] at (0.82,-0.04) {$p_3$};
  \node[pcap] at (0.2,-1.32) {$t=3$:\ fresh direction};
\end{scope}
\begin{scope}[shift={(9.0,0)}]
  \draw[ax] (-1.2,0)--(1.6,0);  \draw[ax] (0,-1.2)--(0,1.35);
  \draw[ell,rotate around={13.7:(0.908,0.392)}] (0.908,0.392) ellipse (0.489 and 0.824);
  \draw[reply] (0.731,0.569) -- (1.085,0.923); \node[lbl,anchor=south west] at (1.08,0.92) {$v_4$};
  \draw[mv] (0.908,0.392) -- (0.980,0.637);
  \fill (0.908,0.392) circle (0.033);    \node[lbl,anchor=north west] at (0.94,0.36) {$p_4$};
  \node[pcap] at (0.2,-1.32) {$t=4$:\ metric now tilted};
\end{scope}
\end{tikzpicture}
\caption{Four rounds in the plane. The shaded set is $\{u:\|u\|_{H_t}\le\alpha\}$; steps
are large in unexplored directions and short in explored ones, with the preconditioner bending the
step away from $v_t$. Drawn with $\alpha=\tau=1$.}
\label{fig:metric}
\end{figure}

\paragraph{Analysis idea.} We track the distance from the current query to the hidden point in the
local metric, $D_t=\|p_t-\wstar\|_{H_t}$. Two effects compete: the point $p_t$ moves closer to
$\wstar$ along $g_t$, which decreases $D_t$, but the metric also stretches along that same
direction, which inflates $D_t$ since the same Euclidean gap counts for more in the new metric.
In the dynamics of $D_t^{2}$ (Lemma~\ref{lem:contract}) the two effects combine into a clean
inequality tied directly to the per-round loss:
\[
  D_{t+1}^{2}\;\le\; D_t^{2}\;-\;\alpha\, r_t\;+\;\alpha^{2}\,s_t^{2}.
\]
Telescoping eliminates $D_t$ and gives
$\Reg_T\le D_1^{2}/\alpha+\alpha\sum_{t\le T}s_t^{2}$,
so all that remains is to bound $\sum_ts_t^{2}$.
This is where our second innovation enters: a potential based on the \emph{trace power}
$\tr\bigl(H_t^{-1/2}\bigr)$, which starts at $d$ and, by Lemma~\ref{lem:var}, decreases by at least
$\tfrac{\tau}{4}s_t^{2}$ each round. Hence
\[
  \sum_{t}s_t^{2}\;\le\;\frac{4d}{\tau}\qquad\text{for every horizon:}
\]
a budget of order $d$ on the total squared step length, with no $T$ in it. The standard potential
for a metric of this shape is $\log\det H_t$, which bounds the same sum but grows like $d\log T$;
that logarithm is the $\ln T$ of \cite{SakaueONS} and of every other row of Table~\ref{tab:oilo}
resting on a volumetric potential. A trace power is bounded outright. Combining the two bounds
gives Theorem~\ref{thm:main}.

\paragraph{Trace powers elsewhere.} The family $\tr(H^{-\beta})$ occurs in several other contexts.
In optimal design it is Kiefer's family of
$\Phi_r$ criteria \cite{Kiefer74}, $\Phi_r(C)=(\tfrac1d\tr C^{r})^{1/r}$ for an information matrix
$C$, with $r=0$ giving D-optimality --- the $\log\det$ potential above --- and $r=-1$ giving
A-optimality; ours is $\Phi_{-1/2}$ at $C=H_t$, strictly between the two. Allen-Zhu, Liao and
Orecchia \cite{ALO15} replace an entropy regularizer by the trace power $-2\tr(A^{1/2})$ to recover
the linear-size spectral sparsifiers of Batson, Spielman and Srivastava \cite{BSS}, removing a
logarithmic factor --- the same trade made here, where $\log\det$ leaves an $O(d\log T)$ and
$\tr(H^{-1/2})$ removes the $\log T$ altogether. Trace powers with $\beta<1$ are also the Schatten
quasi-norms used as surrogates for rank, and the Lewis weights \cite{CP15} behind the $\ell_p$
analogues of the John ellipsoid. The classical cutting-plane potentials, by contrast --- volume for
the ellipsoid and center-of-gravity methods, the volumetric barrier for Vaidya's method
\cite{Vaidya}, and hybrids in the fastest current methods \cite{LSW} --- all sit at $\beta\to0$.

\section{Regret of the Variable Metric Algorithm}\label{sec:alg}

We prove Theorem~\ref{thm:main} for the following algorithm.

\begin{quote}
\textbf{Algorithm VM$(\alpha,\tau)$}. Set $p_1=0$ and $H_1=I$. At round
$t$, query $p_t$; on receiving $v_t$, put $g_t=-v_t$ and update $H_{t+1}$ and $p_{t+1}$ by
\eqref{eq:vm}.
\end{quote}

Write the loss of round $t$ as $r_t=\ip{p_t-\wstar}{g_t}=\ip{\wstar-p_t}{v_t}\ge0$. The proof comes
in two halves. Section~\ref{sec:contract} shows that each step contracts the distance to $\wstar$ by
an amount proportional to $r_t$, at a cost of $\alpha^{2}s_t^{2}$; Section~\ref{sec:pot} shows that
those costs total at most $4d/\tau$ however long the run. Section~\ref{sec:bound} fixes the
constants and puts the two together.

\subsection{Distance contraction}\label{sec:contract}

Everything turns on the distance from the current query to the hidden point, measured in the local
metric: $D_t=\|p_t-\wstar\|_{H_t}$. Each step decreases $D_t^{2}$ by an amount proportional to the
loss, at a cost controlled by $s_t^{2}$. This holds under one condition, $\tau D_t\le\alpha$, which
we verify by induction in Section~\ref{sec:bound}.

\begin{lemma}[Distance contraction]\label{lem:contract}
If $\tau D_t\le\alpha$ then
$D_{t+1}^{2}\le D_t^{2}-\alpha r_t+\alpha^{2}s_t^{2}$.
\end{lemma}

\begin{proof}
Expanding the square in the $H_{t+1}$ metric,
\[
  D_{t+1}^{2}\;=\;\bigl\|p_t-\alpha H_{t+1}^{-1}g_t-\wstar\bigr\|_{H_{t+1}}^{2}
  \;=\;\|p_t-\wstar\|^{2}_{H_{t+1}}-2\alpha r_t+\alpha^{2}\|g_t\|^{2}_{H_{t+1}^{-1}} .
\]
Here $H_{t+1}^{-1}\preceq H_t^{-1}$ bounds $\|g_t\|^{2}_{H_{t+1}^{-1}}\le s_t^{2}$, and the
definition of $H_{t+1}$ expands
$\|p_t-\wstar\|^{2}_{H_{t+1}}=D_t^{2}+\frac{\tau}{s_t}r_t^{2}$. Cauchy--Schwarz in the $H_t$ metric
gives $r_t\le\|p_t-\wstar\|_{H_t}\|g_t\|_{H_t^{-1}}=D_ts_t$, so, as $r_t\ge0$,
\[
  \frac{\tau}{s_t}r_t^{2}\;=\;\tau r_t\cdot\frac{r_t}{s_t}\;\le\;\tau D_tr_t .
\]
Collecting, $D_{t+1}^{2}\le D_t^{2}-(2\alpha-\tau D_t)r_t+\alpha^{2}s_t^{2}$, and $\tau D_t\le\alpha$
makes the bracket at least $\alpha$.
\end{proof}

Whenever the precondition $\tau D_t\le\alpha$ holds for all $t\le T$, telescoping the lemma gives
\begin{equation}\label{eq:tele}
  \alpha\sum_{t\le T}r_t\;\le\;D_1^{2}+\alpha^{2}\sum_{t\le T}s_t^{2},
\end{equation}
so the regret $\Reg_T=\sum_tr_t$ is controlled by the initial distance $D_1$ and the total squared
step length $\sum_ts_t^{2}$. The next subsection bounds $\sum_ts_t^{2}$ by $4d/\tau$, with no
dependence on $T$.

\subsection{The metric potential}\label{sec:pot}


The bound on $\sum_ts_t^{2}$ rests on one inequality, the first-order form of the concavity of
$X\mapsto\tr f(X)$ for concave $f$.

\begin{lemma}[Klein's inequality; see e.g.\ \cite{Carlen10}]\label{lem:klein}
Let $f$ be concave on $[0,\infty)$ and differentiable on $(0,\infty)$, and let $X\succ0$,
$Y\succeq0$. Then
\[
  \tr f(Y)\;\le\;\tr f(X)+\tr\bigl(f'(X)(Y-X)\bigr).
\]
\end{lemma}

\begin{proof}
Write $X=\sum_ix_iu_iu_i^{\top}$ and $Y=\sum_jy_jw_jw_j^{\top}$ spectrally, with $\{u_i\}$ and
$\{w_j\}$ orthonormal bases, and set $s_{ij}=\ip{u_i}{w_j}^{2}$, a doubly stochastic matrix. Then
$\tr f(Y)=\sum_{i,j}s_{ij}f(y_j)$, $\tr f(X)=\sum_{i,j}s_{ij}f(x_i)$ and
$\tr\bigl(f'(X)(Y-X)\bigr)=\sum_{i,j}s_{ij}f'(x_i)(y_j-x_i)$, so the claim reads
\[
  \sum_{i,j}s_{ij}\Bigl[f(y_j)-f(x_i)-f'(x_i)(y_j-x_i)\Bigr]\;\le\;0,
\]
which holds term by term because $f$ lies below its tangents.
\end{proof}

We can now bound the total squared step length. The potential is the trace power
$\tr\bigl(H_t^{-1/2}\bigr)$: it starts at $d$, never goes negative, and Klein's inequality applied to
$f(x)=x^{1/2}$ makes it drop by at least $\tfrac{\tau}{4}s_t^{2}$ each round. The whole sequence of
steps is paid for out of that initial budget of $d$, whatever the horizon.

\begin{lemma}[Step-size bound]\label{lem:var}
For every $\tau\in(0,1]$ and every $T$, $\ \sum_{t\le T}s_t^{2}\le\dfrac{4d}{\tau}$.
\end{lemma}

\begin{proof}
Since $H_{t+1}\succeq H_t\succeq I$ we have $0\prec H_{t+1}^{-1}\preceq H_t^{-1}\preceq I$, and as
$\|g_t\|_2=1$, $s_t^{2}=g_t^{\top}H_t^{-1}g_t\le\lambda_{\max}(H_t^{-1})\le1$. By
Sherman--Morrison applied to $H_{t+1}=H_t+\frac{\tau}{s_t}g_tg_t^{\top}$, and using
$g_t^{\top}H_t^{-1}g_t=s_t^{2}$,
\[
  H_{t+1}^{-1}=H_t^{-1}-\frac{\tau}{s_t+\tau s_t^{2}}\,H_t^{-1}g_tg_t^{\top}H_t^{-1} .
\]
Apply Lemma~\ref{lem:klein} with $f(x)=x^{1/2}$, $X=H_t^{-1}$ and $Y=H_{t+1}^{-1}$, so that
$f'(X)=\tfrac12H_t^{1/2}$:
\[
  \tr\bigl(H_t^{-1/2}\bigr)-\tr\bigl(H_{t+1}^{-1/2}\bigr)\;\ge\;
  \tfrac12\tr\bigl((H_t^{-1}-H_{t+1}^{-1})H_t^{1/2}\bigr)
  \;=\;\frac{\tau}{2(s_t+\tau s_t^{2})}\;g_t^{\top}H_t^{-3/2}g_t .
\]
Write $H_t^{-1}=\sum_i\mu_iq_iq_i^{\top}$ spectrally and set $w_i=\ip{g_t}{q_i}^{2}$, which sum to
$1$. Convexity of $x\mapsto x^{3/2}$ then gives
$g_t^{\top}H_t^{-3/2}g_t=\sum_iw_i\mu_i^{3/2}\ge\bigl(\sum_iw_i\mu_i\bigr)^{3/2}=s_t^{3}$. Hence
\[
  \tr\bigl(H_t^{-1/2}\bigr)-\tr\bigl(H_{t+1}^{-1/2}\bigr)
  \;\ge\;\frac{\tau s_t^{3}}{2(s_t+\tau s_t^{2})}=\frac{\tau s_t^{2}}{2(1+\tau s_t)}
  \;\ge\;\frac{\tau}{4}\,s_t^{2},
\]
using $s_t\le1$ and $\tau\le1$. Telescoping, with $\tr(H_1^{-1/2})=d$ and
$\tr(H_{T+1}^{-1/2})\ge0$, gives the claim.
\end{proof}

\begin{remark}[Importance of self-normalization]\label{rem:onedim}
The effect of self-normalization is already visible in one dimension. When $d=1$ the metric is a
scalar $h_t$ and $s_t=h_t^{-1/2}$. With the self-normalized update, \eqref{eq:vm} reads
$h_{t+1}=h_t+\tau\sqrt{h_t}$, so $\sqrt{h_t}$ grows linearly:
$\sqrt{h_t}\approx1+\tfrac{\tau}{2}(t-1)$. Then $s_t^{2}\asymp t^{-2}$ and
$\sum_ts_t^{2}\approx2/\tau$ converges --- the budget is finite. Without self-normalization, the
plain update $h_{t+1}=h_t+\tau$ gives $h_t\approx\tau t$, so
$\sum_{t\le T}s_t^{2}\approx\tau^{-1}\ln T$, which diverges. The entire difference between
$O(d\log T)$ and $O(d)$ comes down to the difference between $\sum1/t$ and $\sum1/t^{2}$.
\end{remark}

\subsection{The regret bound}\label{sec:bound}

It remains to choose $\alpha$ and $\tau$. The trade-off is visible in \eqref{eq:tele}: a large
$\alpha$ shrinks $D_1^{2}/\alpha$ but inflates $\alpha\sum_ts_t^{2}$, and $\tau$ must stay small
enough relative to $\alpha$ for the precondition $\tau D_t\le\alpha$ of Lemma~\ref{lem:contract} to
hold at every step. The values below balance the two.

\begin{theorem}\label{thm:vm}
Let $R=\max_{w\in K_1}\|w\|$ and run VM$(\alpha,\tau)$ with
$\alpha=\frac{R}{2\sqrt3\,d}$ and $\tau=\frac{1}{6d}$. Then $\Reg_T\le6\sqrt3\,dR\le11\,dR$ for every
$T$, every $\wstar\in K_1$ and every strong separation oracle. This is Theorem~\ref{thm:main}; for
$K_1=\Ball$ it reads $\Reg_T\le11\,d$.
\end{theorem}

\begin{proof}
We first show $D_t\le\sqrt3\,R$ for every $t$, by induction. At $t=1$, $H_1=I$ and $p_1=0$ give
$D_1=\|\wstar\|\le R$. Whenever $D_k\le\sqrt3R$ we have
$\tau D_k\le\frac{\sqrt3R}{6d}=\frac{R}{2\sqrt3\,d}=\alpha$, so Lemma~\ref{lem:contract} applies at
step $k$. Assume the bound for all $k\le t$ and telescope the lemma over those steps; dropping the
nonpositive term $-\alpha\sum_kr_k$ leaves
\[
  D_{t+1}^{2}\;\le\;D_1^{2}+\alpha^{2}\sum_{k\le t}s_k^{2}
  \;\le\;R^{2}+\frac{R^{2}}{12d^{2}}\cdot 24d^{2}\;=\;3R^{2},
\]
by Lemma~\ref{lem:var} with $\frac{4d}{\tau}=24d^{2}$, legitimate since $\tau=\frac{1}{6d}\le1$.
This closes the induction. Telescoping once more and keeping the regret term,
\[
  \alpha\sum_{t\le T}r_t\;\le\;D_1^{2}-D_{T+1}^{2}+\alpha^{2}\sum_{t\le T}s_t^{2}\;\le\;R^{2}+2R^{2}=3R^{2},
\]
so $\Reg_T=\sum_tr_t\le\frac{3R^{2}}{\alpha}=6\sqrt3\,dR$.
\end{proof}

\begin{remark}[Cost]\label{rem:cost}
Maintaining $H_t^{-1}$ by Sherman--Morrison makes the metric update and the vector $H_{t+1}^{-1}g_t$
cost $O(d^{2})$ per round: there is no eigendecomposition, no optimization subproblem, and nothing
that depends on $T$. The iterates stay bounded on their own, since $H_t\succeq I$ gives
$\|p_t-\wstar\|_2\le D_t\le\sqrt3\,R$ and hence $\|p_t\|\le(1+\sqrt3)R$.
\end{remark}

\section{Further Properties}\label{sec:further}

We collect three properties of the game and of the algorithm. Section~\ref{sec:lower} shows that
every algorithm suffers $\Reg\ge\sqrt d\,R$, so Theorem~\ref{thm:vm} is optimal up to a factor
$\sqrt d$. Section~\ref{sec:corruption} lets the oracle lie on a total budget $C$ and pays $O(C)$
for it. Section~\ref{sec:lowrank} replaces $d$ by the rank of the responses actually received, without
knowing that rank in advance.

\subsection{Lower bound}\label{sec:lower}

The following lower bound is due to Sakaue, Tsuchiya, Bao and Oki \cite{SakaueONS}; we include the
short proof for completeness.

\begin{theorem}[\cite{SakaueONS}]\label{thm:lower}
Suppose $K_1$ is invariant under sign flips in some orthonormal basis $u_1,\dots,u_d$. Then every
algorithm, deterministic or randomized, satisfies
\[
  \Reg_d\;\ge\;\max_{w\in K_1}\ \sum_{j=1}^{d}\bigl|\ip{w}{u_j}\bigr| .
\]
For the unit ball this is $\sqrt d$ and for the cube $[-1,1]^{d}$ it is $d$; in both cases it equals
$\sqrt d\,R$.
\end{theorem}

\begin{proof}
Fix $w=\sum_jc_ju_j\in K_1$ attaining the maximum, with all $c_j\ge0$, which sign symmetry permits.
For the unit ball, this is $c_j=1/\sqrt d$ for every $j$, giving $\sum_jc_j=\sqrt d$.
At round $j\le d$ the adversary observes $p_j$ and returns
$v_j=\sigma_ju_j$ with $\sigma_j=-\operatorname{sign}\ip{p_j}{u_j}$, committing to
$\ip{\wstar}{u_j}=\sigma_jc_j$. The point $\wstar=\sum_j\sigma_jc_ju_j$ is a sign flip of $w$, hence
lies in $K_1$, and every response is consistent with it:
\[
  \ip{\wstar-p_j}{v_j}=\sigma_j\bigl(\sigma_jc_j-\ip{p_j}{u_j}\bigr)
  =c_j+\bigl|\ip{p_j}{u_j}\bigr|\;\ge\;c_j\;\ge\;0 .
\]
Summing the $d$ losses gives the claim. The adversary is adaptive and the argument is pathwise, so it
applies verbatim to randomized learners.
\end{proof}

Since $\max_{U\in O(d)}\|Uw\|_1=\sqrt d\,\|w\|_2$, the quantity in Theorem~\ref{thm:lower} never
exceeds $\sqrt d\,R$, with equality exactly when the extreme point of $K_1$ of largest norm has
equal-magnitude coordinates in the good basis --- as for the ball and the cube. Together with
Theorem~\ref{thm:vm}, the optimal regret on the unit ball therefore lies between $\Omega(\sqrt d)$
and $O(d)$.

\begin{remark}[Richer feedback]\label{rem:slack}
The proof never uses that the learner is denied the loss: the adversary fixes $\ip{\wstar}{u_j}$
only after seeing $p_j$, so the bound holds verbatim when $r_t$ is revealed at the end of each
round. The hardness is not informational but a guess-before-reveal budget.
\end{remark}

Under that richer feedback the bound is tight, so $\sqrt d\,R$ is the right target for the game as
stated too.

\begin{proposition}\label{prop:slack}
Suppose $r_t$ is revealed each round, and put $c_s:=r_s+\ip{p_s}{v_s}=\ip{\wstar}{v_s}$. The
minimum-norm rule $p_t=\argmin\{\|w\|:\ip{w}{v_s}=c_s,\ s<t\}$ satisfies $\Reg_T\le\sqrt d\,R$ for
every $T$ and every adversary. With Theorem~\ref{thm:lower} the value of this game is therefore
exactly $\sqrt d\,R$ for the ball and the cube.
\end{proposition}

\begin{proof}
Observing $r_t$ turns the round's information from a halfspace into the hyperplane
$\ip{\wstar}{v_t}=c_t$, so the feasible set is the affine subspace $\wstar+V_{t-1}^{\perp}$, where
$V_{t-1}=\spn(v_1,\dots,v_{t-1})$, and the rule --- computable from the observed $(v_s,c_s)$ alone
--- returns $p_t=\Pi_{V_{t-1}}\wstar$. Put $x_t=\wstar-p_t=\Pi_{V_{t-1}^{\perp}}\wstar$ and split
$v_t=u+z$ with $u\in V_{t-1}$ and $z\in V_{t-1}^{\perp}$. Then $r_t=\ip{x_t}{z}$, and if $z\ne0$
then $V_t=V_{t-1}\oplus\spn(z)$, so
\[
  \|x_{t+1}\|^{2}\;=\;\|x_t\|^{2}-\frac{r_t^{2}}{\|z\|^{2}}\;\le\;\|x_t\|^{2}-r_t^{2},
\]
using $\|z\|\le\|v_t\|=1$; if $z=0$ then $r_t=0$ and $x_{t+1}=x_t$, so the inequality holds anyway.
Telescoping gives $\sum_tr_t^{2}\le\|\wstar\|^{2}\le R^{2}$. Finally, $r_t>0$ forces $z\ne0$, hence
$\dim V_t>\dim V_{t-1}$, so at most $d$ rounds have $r_t>0$; Cauchy--Schwarz over those rounds gives
$\Reg_T\le\sqrt{d}\,R$.
\end{proof}

\subsection{Corruption Robustness}\label{sec:corruption}

The reduction \eqref{eq:reduction} needs the expert to be at least as good as the learner under the
expert's own criterion on \emph{every} round, and nothing guarantees that. We therefore let the
oracle lie, at a price. Write $x_+=\max(0,x)$ and $x_-=\max(0,-x)$.

\begin{definition}[$C$-corrupted oracle]\label{def:corrupt}
The oracle is \emph{$C$-corrupted} if it returns unit vectors $v_t$ subject only to
\[
  r_t\;=\;\ip{\wstar-p_t}{v_t}\;\ge\;-\epsilon_t,\qquad\epsilon_t\ge0,\qquad\sum_t\epsilon_t\;\le\;C,
\]
that is, $\sum_t(r_t)_-\le C$. The regret is $\Reg_T:=\sum_{t\le T}(r_t)_+$, the loss suffered on
the rounds whose feedback was legitimate.
\end{definition}

A corrupted round of inverse linear optimization is one on which the expert's own criterion ranks
the action it took \emph{below} the learner's recommendation. The count model, at most $C_0$
arbitrary rounds, is the case $C=O(C_0R)$, since the induction below keeps
$|r_t|\le D_ts_t\le\sqrt3R$. The algorithm itself is unchanged: the one sign-sensitive step of
Lemma~\ref{lem:contract} was $\frac{\tau}{s_t}r_t^{2}\le\tau D_tr_t$, and its absolute-value form
costs nothing.

\begin{lemma}[Corrupted contraction]\label{lem:contractC}
If $\tau D_t\le\alpha$ then
$\;D_{t+1}^{2}\le D_t^{2}-\alpha(r_t)_++3\alpha(r_t)_-+\alpha^{2}s_t^{2}$.
\end{lemma}

\begin{proof}
As in Lemma~\ref{lem:contract},
$D_{t+1}^{2}=D_t^{2}+\frac{\tau}{s_t}r_t^{2}-2\alpha r_t+\alpha^{2}\|g_t\|^{2}_{H_{t+1}^{-1}}$ with
$\|g_t\|^{2}_{H_{t+1}^{-1}}\le s_t^{2}$. Cauchy--Schwarz in the $H_t$ metric gives $|r_t|\le D_ts_t$
\emph{with no sign assumption}, so $\frac{\tau}{s_t}r_t^{2}\le\tau D_t|r_t|\le\alpha|r_t|$. Hence
$D_{t+1}^{2}\le D_t^{2}+\alpha(|r_t|-2r_t)+\alpha^{2}s_t^{2}$, and $|r|-2r=-(r)_++3(r)_-$.
\end{proof}

Lemma~\ref{lem:var} needs no revision: it was proved for \emph{arbitrary} unit vectors $g_t$, so the
budget $\sum_ts_t^{2}\le4d/\tau$ survives verbatim and only the tuning must absorb the corruption.

\begin{theorem}\label{thm:corrupt}
Let $\Delta=\sqrt3R$ and, given any upper bound $C$ on the corruption, run VM$(\alpha,\tau)$ with
$\alpha=\min\bigl\{\frac{R}{4\sqrt3\,d},\frac{R^{2}}{3C}\bigr\}$ and $\tau=\alpha/\Delta$. Then
$\Reg_T\le8\sqrt3\,dR+6C\le14\,dR+6C$ for every $T$, every $\wstar\in K_1$ and every $C$-corrupted
oracle.
\end{theorem}

\begin{proof}
First $\tau=\alpha/(\sqrt3R)\le\frac{1}{12d}\le1$, so Lemma~\ref{lem:var} applies. We show
$D_t\le\Delta$ by induction; at $t=1$, $D_1=\|\wstar\|\le R\le\Delta$. If $D_k\le\Delta$ for $k\le t$
then $\tau D_k\le\tau\Delta=\alpha$, so Lemma~\ref{lem:contractC} applies at each such step, and
telescoping while dropping $-\alpha(r_k)_+\le0$,
\[
  D_{t+1}^{2}\;\le\;R^{2}+3\alpha\!\!\sum_{k\le t}(r_k)_-+\alpha^{2}\frac{4d}{\tau}
  \;\le\;R^{2}+3\alpha C+4\sqrt3\,d\alpha R\;\le\;3R^{2}=\Delta^{2},
\]
the last two terms being at most $R^{2}$ each by the two arms of the minimum defining $\alpha$.
Telescoping once more and keeping the regret term gives
$\Reg_T\le\frac{R^{2}}{\alpha}+3C+4\sqrt3\,dR$, and
$\frac{R^{2}}{\alpha}=\max\{4\sqrt3\,dR,3C\}\le4\sqrt3\,dR+3C$.
\end{proof}

Only the \emph{scale} of $(\alpha,\tau)$ changes, their ratio staying $\sqrt3R$. The dependence on
$C$ is optimal against the $\Omega(\sqrt d\,R+C)$ lower bound (Theorem~\ref{thm:lower} with the
trivial $\Omega(C)$). Through \eqref{eq:transfer} with $K_1=\Ball$, a $C$-corrupted expert leaves
the induced learner with $R_T\le\tilde R_T=O(d+C)$ for every $T$. Some robustness is even
free: merely halving $(\alpha,\tau)$ in Theorem~\ref{thm:vm} absorbs any $C=O(dR)$
with no knowledge of $C$ at all.

\paragraph{Comparison.} The corruption model in this form descends from Lykouris, Mirrokni and Paes
Leme \cite{LMPL18} for stochastic bandits, and was brought to contextual search by Krishnamurthy,
Lykouris, Podimata and Schapire \cite{KLPS23}, who obtained $O(d^{3}\log^{3}T+C\log^{2}T)$ for the
symmetric loss. Paes Leme, Podimata and Schneider \cite{PPS26} then reached $O(d\log T+C)$ there and
an optimal $O(d\log(1/\varepsilon)+C)$ for the $\varepsilon$-ball loss, asking in closing whether
$O(d+C)$ is attainable. Theorem~\ref{thm:corrupt} is the analogous $O(dR+C)$ for the game of
Section~\ref{sec:intro}, uniform in $T$; robustness to a suboptimal expert has also been studied in
inverse linear optimization itself \cite{SakaueONS,OkiSakaue}. The catch is that both \cite{KLPS23}
and \cite{PPS26} are agnostic to $C$, whereas our tuning needs an upper bound on the budget.

That dependence is not an artifact: the finiteness of $\sum_ts_t^{2}$ --- the very fact that buys
$O(d)$ in place of $O(d\log T)$ --- caps the distance $\alpha\sum_ts_t^{2}\le4d\alpha/\tau$ the query
can travel over its whole lifetime, and no step-size schedule removes the cap, so an adversary that
pushes it past the remaining budget may then answer honestly forever at a loss linear in $T$. We
leave it as an open problem to obtain $O(dR+C)$ agnostically, with no upper bound on $C$ given to
the learner; by the above, it cannot come from retuning alone.

\subsection{Low-Rank Responses}\label{sec:lowrank}

Nothing forces the oracle to explore all of $\R^d$. Write $V_t=\spn(v_1,\dots,v_t)$ and
$k=\dim V_T$; if the action sets lie in a common low-dimensional affine subspace, say, then $k$ is
small and we should expect a bound of order $k$. The dimension enters Section~\ref{sec:alg} in two
places only. In the first --- the initial value $\tr(H_1^{-1/2})=d$ of the potential --- adaptivity
is free.

\begin{lemma}[Rank-adaptive step-size bound]\label{lem:varlow}
For every $\tau\in(0,1]$ and every $T$, $\ \sum_{t\le T}s_t^{2}\le\frac{4}{\tau}\,d_T$, where
$d_T:=d-\tr\bigl(H_{T+1}^{-1/2}\bigr)$ satisfies $d_T\le\min(k,d)$.
\end{lemma}

\begin{proof}
The proof of Lemma~\ref{lem:var} bounds each drop $\tr(H_t^{-1/2})-\tr(H_{t+1}^{-1/2})$ below by
$\frac{\tau}{4}s_t^{2}$; telescoping from $\tr(H_1^{-1/2})=d$ gives the bound, and $d_T\le d$ since
$H_{T+1}^{-1/2}\succ0$. Writing $H_{T+1}=I+M$ with $M\succeq0$ and $\range(M)\subseteq V_T$, the
subspace $V_T^{\perp}$ is invariant and $H_{T+1}$ acts on it as the identity, so $d-k$ eigenvalues of
$H_{T+1}^{-1/2}$ equal $1$ and the rest are positive: $\tr(H_{T+1}^{-1/2})\ge d-k$.
\end{proof}

The quantity $d_T$ is a soft rank: at most $k$, smaller when the responses merely cluster near a
subspace, and computable from $H_{T+1}$. Theorem~\ref{thm:vm} used Lemma~\ref{lem:var} only through
the bound on $\sum_ts_t^{2}$, so its proof now gives more.

\begin{corollary}\label{cor:tunedm}
If $k\le m$ then VM$\bigl(\frac{R}{2\sqrt3\,m},\frac{1}{6m}\bigr)$ satisfies $\Reg_T\le6\sqrt3\,mR$
for every $T$.
\end{corollary}

\begin{proof}
Verbatim the proof of Theorem~\ref{thm:vm}: the budget is now $4k/\tau\le4m/\tau=24m^{2}$, so the
induction closes at $D_t^{2}\le R^{2}+\frac{R^{2}}{12m^{2}}\cdot24m^{2}=3R^{2}$ and
$\Reg_T\le3R^{2}/\alpha=6\sqrt3\,mR$.
\end{proof}

This is no accident. Since $\range(H_{t+1}-I)\subseteq V_t$, the subspace $V_t$ is
$H_{t+1}$-invariant, hence so is $H_{t+1}^{-1}$, so $H_{t+1}^{-1}g_t\in V_t$ and, with $p_1=0$,
$p_t\in V_{t-1}$ for every $t$. The metric stays the identity on $V_t^{\perp}$ and the query never
leaves $V_{t-1}$, so a run against rank-$k$ responses \emph{is} a run of the same algorithm inside
$V_T\cong\R^{k}$ against the hidden point $\Pi_{V_T}\wstar$, of norm at most $R$.

Only the tuning knows $d$, and there the dependence is real. Take $K_1=R\Ball$, $\wstar=Re_1$ and
$v_t=e_1$ for every $t$ --- legal whenever $\ip{p_t}{e_1}\le R$, and of rank $1$. The bounds
$H_t\succeq I$ and $s_t\le1$ give
$\|H_{t+1}^{-1}g_t\|\le\|H_{t+1}^{-1/2}\|\,\|g_t\|_{H_{t+1}^{-1}}\le1$, so the query moves by at most
$\alpha$ per round, $r_t\ge R-\alpha t$, and the first $R/2\alpha$ rounds cost at least $R/2$ each:
$\Reg\ge R^{2}/(4\alpha)=\Omega(dR)$ at the tuning of Theorem~\ref{thm:vm}. A step of order
$R/d$ needs order $d$ rounds merely to cross the initial distance $R$, whatever the rank.

The tuning must therefore adapt, and it can, because the learner \emph{observes} $\dim V_t$. We
double the guess and restart whenever the guess is reached.

\begin{quote}
\textbf{Algorithm VM-A.} For $j=0,1,2,\dots$ put $\kappa_j=\min(2^{j},d)$ and run a fresh copy of
VM$\bigl(\frac{R}{2\sqrt3\,\kappa_j},\frac{1}{6\kappa_j}\bigr)$ from $p=0$, $H=I$, until the
responses received during phase $j$ span $\kappa_j$ dimensions --- if $\kappa_j=d$, forever.
\end{quote}

\begin{theorem}\label{thm:adaptive}
VM-A satisfies $\Reg_T\le6\sqrt3\min(4k,3d)\,R\le42\,kR$ for every $T$, every $\wstar\in K_1$ and
every strong separation oracle, without prior knowledge of $k$.
\end{theorem}

\begin{proof}
A phase queries its own iterates, so the responses it receives are legal for it, and it starts at
$p=0$, $H=I$ with $D_1=\|\wstar\|\le R$: phase $j$ is an instance of the game whose responses span at
most $\kappa_j$ dimensions, and Corollary~\ref{cor:tunedm} bounds its regret by $6\sqrt3\kappa_jR$. A
phase ends only if its own responses span $\kappa_j$ dimensions, which forces $k\ge\kappa_j$; and only the
last phase can have $\kappa_j=d$, so $\kappa_j=2^{j}$ for $j<J$, where $J$ indexes the last phase
started. Then $2^{J-1}\le k$ and $2^{J-1}<d$, whence $\kappa_J\le\min(2k,d)$ and
$\sum_{j\le J}\kappa_j=2^{J}-1+\kappa_J\le\min(4k,3d)$. Sum the phase bounds; $24\sqrt3\le42$.
\end{proof}

Both features of the scheme are forced. The restart must reset the \emph{query}, since each phase
has to begin within distance $R$ of $\wstar$ and $p=0$ is the only point known to do so. Keeping $p$
and resetting only $H$ leaves $D_1\le\sqrt3R$ (Remark~\ref{rem:cost}) and costs a constant factor
per phase; resetting $p$ while keeping $H$ leaves $D_1=\|\wstar\|_{H}$, unbounded. The ratio $2$ is
optimal as well: with ratio $\theta$ the guesses sum to $\frac{\theta^{2}}{\theta-1}k$, least at
$\theta=2$.

\begin{corollary}
If the sets $X_t$ lie in a common affine subspace of dimension $k$ --- more generally, if
$\spn\{x_t-\hat x_t\}$ has dimension $k$ --- then VM-A through the reduction of
Section~\ref{sec:intro} gives $R_T\le\tilde R_T\le84\,k$ for every $T$: the responses
$v_t\propto x_t-\hat x_t$ have rank $k$, and $R=1$.
\end{corollary}

\section{Application to Convex Optimization}\label{sec:convex}

We minimize a convex function from the \emph{directions} of its subgradients alone, with a total
suboptimality that stays bounded over an infinite run.

Let $f:\R^d\to\R$ be convex with a minimizer $\xstar$, $\|\xstar\|\le R$, queried through a
first-order oracle: the learner asks $p_t$ and receives a subgradient $\nabla f(p_t)$. Since
$\ip{\nabla f(p_t)}{p_t-\xstar}\ge f(p_t)-f(\xstar)\ge0$, the normalized direction
$v_t=-\nabla f(p_t)/\|\nabla f(p_t)\|$ is a legal response in the game of Section~\ref{sec:intro}
--- chosen after the query, as a strong oracle may --- with loss
$\ip{\xstar-p_t}{v_t}\ge\bigl(f(p_t)-f(\xstar)\bigr)/\|\nabla f(p_t)\|$. If the oracle ever returns
$\nabla f(p_t)=0$ then $p_t$ already minimizes $f$ and the learner stays there, so every later round
is free and the bounds below are unaffected; we may thus assume $\nabla f(p_t)\neq0$.
Theorem~\ref{thm:main}, with
$\wstar=\xstar$ and $K_1=R\Ball$, therefore bounds $\sum_{t\le T}(f(p_t)-f(\xstar))/\|\nabla f(p_t)\|$
by $O(dR)$ for every $T$, with no Lipschitz assumption. If $f$ is $L$-Lipschitz then
$\|\nabla f\|\le L$, and with $\bar p_T=\frac1T\sum_{t\le T}p_t$ convexity gives
\begin{equation}\label{eq:convex}
  \sum_{t\le T}\bigl(f(p_t)-f(\xstar)\bigr)\;\le\;11\,dLR
  \qquad\text{and}\qquad
  f(\bar p_T)-f(\xstar)\;\le\;\frac{11\,dLR}{T}.
\end{equation}
The left bound is finite at $T=\infty$: the total suboptimality ever incurred is bounded, not merely
its average. The method reads only subgradient \emph{directions} --- it never evaluates $f$, and
needs neither $L$ nor $f(\xstar)$ --- and costs $O(d^{2})$ per round. Its iterates stay in a ball of
radius $O(R)$ (Remark~\ref{rem:cost}). To minimize over a convex $K\ni\xstar$ we insert the
projection $\Pi^{H_{t+1}}_{K}$, nonexpansive in $\|\cdot\|_{H_{t+1}}$, which leaves
Lemma~\ref{lem:contract} intact.

Against subgradient descent the trade is $dLR/T$ for $LR/\sqrt T$: \eqref{eq:convex} is better once
$T\gtrsim d^{2}$, and below $T=d$ nothing improves on $1/\sqrt T$ \cite{NY,Nesterov}. Cutting-plane
methods --- center of gravity and ellipsoid \cite{NY}, Vaidya's method \cite{Vaidya} and its
successors \cite{LSW} --- are
incomparably stronger in $T$, reaching accuracy $\epsilon$ within $O(d\ln(LR/\epsilon))$ queries, but
their guarantee is for the \emph{best} iterate. The queries themselves need not be good, and singling
out the good one requires the values $f(p_t)$, that is, a zero-order oracle. Guarantees for the
average come instead from the algorithms of Table~\ref{tab:oilo}, each of which transfers through the
reduction above, replacing $O(dLR/T)$ in \eqref{eq:convex} by $O(d^{4}LR\ln T/T)$
\cite{Besbes21,Besbes25}, by $O(dLR\ln T/T)$ \cite{GGKMPS,SakaueONS}, or by $\exp(O(d\ln d))\,LR/T$
\cite{GGKMPS}. Ours is the first constant at once polynomial in $d$ and free of $\ln T$.

\end{document}